\documentclass[12pt]{article}

\usepackage{setspace}
\usepackage{xcolor}
\usepackage[final]{graphicx}
\usepackage{amsfonts}
\usepackage{bbm}
\usepackage{amsthm}
\usepackage{listings}
\usepackage{amsmath, amssymb}
\definecolor{dkgreen}{rgb}{0,0.6,0}
\definecolor{gray}{rgb}{0.5,0.5,0.5}
\definecolor{mauve}{rgb}{0.58,0,0.82}
\usepackage{fancyhdr}
\usepackage{url}
\usepackage{enumitem}
\usepackage{multirow}
\usepackage{shuffle}
\usepackage{lastpage}
\usepackage[font=small,labelfont=bf]{caption}
\usepackage{float}
\usepackage[caption = false]{subfig}
\usepackage{slashbox}
\usepackage{bm} 

\usepackage{arxiv}

\usepackage{afterpage}
\usepackage[pdfpagelabels]{hyperref}
	\hypersetup{
		colorlinks   = true,
		linkcolor = blue,
		urlcolor    = blue
		}

\newtheorem{defn}{Definition}[section]
\newtheorem{thm}{Theorem}[section]
\newtheorem{lem}{Lemma}[section]
\newtheorem{rem}{Remark}[section]
\newtheorem{prop}{Proposition}[section]
\newtheorem{exam}{Example}[section]

\newcommand{\E}{\mathbb{E}}
\newcommand{\p}{\mathbb{P}}

\newcommand{\R}{\mathbb{R}}

\newcommand{\V}{\mathcal{V}}
\newcommand{\N}{\mathbb{N}}

\DeclareMathOperator*{\argmax}{arg\,max}

\numberwithin{equation}{section}

\graphicspath{{./figures/} }

\title{Convolutional Signature for Sequential Data
\thanks{\, Work is supported by NSF grant DMS-2008427 for the second author.}}
\fancypagestyle{plain}{
\fancyhead{}

}

\author{\and Ming Min \thanks{Department of Statistics and Applied Probability, South Hall, University of California, Santa Barbara, CA 93106, USA (E-mail: \href{mailto:m_min@pstat.ucsb.edu}{m\_min@pstat.ucsb.edu}).}  \and Tomoyuki Ichiba \thanks{Department of Statistics and Applied Probability, South Hall, University of California, Santa Barbara, CA 93106, USA (E-mail: \href{mailto:ichiba@pstat.ucsb.edu}{ichiba@pstat.ucsb.edu}).}
}

\date{\vspace{-5ex}}

\begin{document}




\maketitle

\begin{abstract}
Signature is an infinite graded sequence of statistics known to characterize geometric rough paths. While the use of the signature in machine learning is successful in low dimensional cases, it suffers from the curse of dimensionality in high dimensional cases, as the number of features in the truncated signature transform grows exponentially fast. With the idea of Convolutional Neural Network, we propose a novel neural network to address this problem. Our model reduces the number of features efficiently in a data dependent way. Some empirical experiments including high dimensional financial time series classification and natural language processing are provided to support our convolutional signature model. 
\keywords{Signature \and Rough Paths \and  Convolutional Neural Networks \and Sequential Data} 
\end{abstract}



\section{Introduction}
Multi-dimensional sequential data analysis is an important research area in Machine Learning, Financial Mathematics and many other areas. There are several methods of analyzing sequential data recently developed in deep learning, e.g.,  Recurrent Neural Network (RNN) \cite{bff0e6bd8f4a4f0d9735bf1728fb43ef}, GRU \cite{cho2014learning}, LSTM \cite{lstm} and Transformer \cite{NIPS2017_3f5ee243}. They have been successfully applied into a variety of important tasks in Data Science, such as natural language processing, financial time series and medical data analyses. Another mainstream approach to the sequential data analyses is Bayesian learning, mostly involved with Gaussian Process (GP) \cite{williams2006gaussian}, where by pre-determined a priori distribution, it has advantage in quantifying uncertainty up to some extent. For example, \cite{gp_pde} use GP to solve nonlinear partial differential equations with noisy boundary observations. 
More recently, a novel mathematical object, called {\it signature}, has been proposed and received more attention, in order to summarize information of sequential data, see \cite{boedihardjo2014signature,Chevyrev_2016,Levin2013LearningFT,lyons2002system,LyonsTerryJ2007DEDb}. In this paper we shall discuss the signature in the multi-dimensional sequential data analysis.

Signature is a graded feature set of a stream, or sequential data set, which is derived from the Rough Path Theory. Signature has been introduced as a feature map into the field of Machine Learning with successful applications to the sequential data. With  truncations up to a given desirable accuracy,  this special feature set has universality for approximations and can  characterize pathwise data efficiently. It is known that the high frequency sequential data set is transformed into several features efficiently by the truncated signature in the case of relatively low dimensional paths. For instance, 
\cite{lyons_hq} use the signatures to characterize high frequency trading paths of financial data. \cite{NIPS2019_8574} use the signature transform as a layer of neural network, and propose the deep signature transform model. Moreover, the use of the signature is model free and  data dependent, see \cite{lyons2019nonparametric,lyons2019numerical}.

However,  the application of the signatures suffers from the curse of dimensionality,  because the number of features in the truncated signature grows exponentially fast as the dimension increases. Consequently, in the case of high dimensional sequential data, this feature map requires computational costs in real data analyses. A kernel based learning algorithm has been introduced to address this problem in  \cite{JMLR:v20:16-314,toth2019bayesian}. In this paper, we propose a new algorithm to solve this problem by combining Convolutional Neural Network and the signature transform. We evaluate the reduction of the number of features and show by numerical experiments that this algorithm can gain efficiency. Thus this algorithm may contribute many applications of the signatures in the cases of the high dimensional sequential data.

The rest of this paper is organized as follows. In Section \ref{sec_sig}, we review the signature of  rough paths, geometric rough paths and nice properties, discuss Signature Classifier in the classification problems, as a typical application of the signature, and evaluate its classification error in Theorem \ref{thm_concentration}. In Section \ref{sec_cnnsig}, we introduce the main algorithm of this paper, a Convolutional Signature model, evaluate how this model reduces the number of features, show that this model preserves all information of path data and discuss its universality in Theorem \ref{thm:App-CNN-Sig}. In Section \ref{sec_exp}, a broad range of experiments are performed to support our model, including high dimensional, financial time series classification, functional estimation and textual sentimental detection. We conclude with further ongoing research in Section \ref{sec_conc}. 

\section{Signature and Geometric Rough Paths}
\label{sec_sig}
\subsection{Signatures}
Let us introduce some notations for the sequential data sets, in order to explain the signature method, following  \cite{LyonsTerryJ2007DEDb}. Given a Banach space $E$ with a norm $\|\cdot\|$, we define the tensor algebra 
	\begin{equation} \label{eq: whole}
	T((E)) := \{(a_i)_{i\geq 0}: a_i\in E^{\otimes i} \text{ for every }i \} 
	\end{equation} 
associated with the sum $+$ and with the tensor product $\otimes$ defined by 
\begin{equation*}
\begin{split}
(a_i)_{i\geq 0} + (b_i)_{i\geq 0} :=& \, (a_i+b_i)_{i\geq 0}, \quad 
(a_i)_{i\geq 0} \otimes (b_i)_{i\geq 0} := \, (c_i)_{i\geq 0}, \ 
\end{split}
\end{equation*}
where the $j$th element $c_j := \sum_{k=0}^j a_k \otimes b_{j-k}$ is the convolution  of the first $j$ elements of $ (a_{i})_{i \ge 0}$ and $(b_{i})_{i\ge 0 }$ in $T((E))$. 
%
Similarly, let us define its subset 
\begin{equation} \label{eq: finite space}
T(E) := \{(a_i)_{i\geq 0}: a_i\in E^{\otimes i} \text{ and } \exists N \in \mathbb N \text{ such that } a_i=0 \ \forall i\geq N \} \end{equation}
of $T((E))$ for those with finite number of non-zero elements. Note that $\,T(E) \subset T((E)) \,$. 
Also, we shall consider the truncated tensor algebra of order $m\in \mathbb{N}$, i.e., 
\begin{equation} \label{eq: truncatedbym}
T^m(E) := \{(a_i)_{i = 0}^m: a_i\in E^{\otimes i} \text{ for }\forall i\leq m \},
\end{equation}
which is a subalgebra of $T((E))$. 
Then as we shall see, the signatures and the $m$-th order truncated  signatures lie in these spaces $T((E))$ and $T^m(E)$, respectively. 

Now with $E :=  \R^{d}$ and the usual Euclidean norm $ \lVert \cdot \rVert$, we shall define the space $\V^p([0,T], E)$ of the $d$-dimensional continuous paths of finite $p$-th variation over the time interval $[0, T]$ and the signatures of the paths in $\V^p([0,T], E)$. 

\begin{defn}[The space of finite $p$-variation paths]
\label{bdd_var}
Fix $p\geq 1$ and the interval $\,[0, T]\,$. The $p$-variation of a $d$-dimensional path $X: [0, T]\to E:= \R^d$  is defined by
\begin{equation*}\| X \|_{p} := \left( \sup_{D_{n}\subset [0, T]} \sum_{i=0}^{n-1}\| X_{t_{i+1}} - X_{t_i} \|^p \right)^{1/p}.\end{equation*}
Here, the supremum is taken over all the possible partitions of the form $D_{n}:=\{t_i\}_{1\leq i\leq n}$ of $[0,T]$ with $0 = t_0<t_1<\cdots<t_n \le T$, $\, n \ge 1\,$.   
$X$ is said to be of finite $p$-variation, if $\| X \|_p < \infty$. We denote the set of continuous paths $X:[0,T] \to E$ of finite $p$-variation by $\V^p([0,T], E)$.
\end{defn}

We use the supremum norm $ \lVert \cdot \rVert_{\infty}$ for continuous functions on $\,[0, T]\,$, i.e., $ \lVert f \rVert_{\infty} := \sup_{x \in [0, T]} \lvert f(x) \rvert$. It can be shown that if we equip the space $\V^p([0,T], E)$ with the norm $\|X\|_{\V^p([0,T], E)} := \|X\|_p + \|X\|_{\infty}$, then $\V^p([0,T], E)$ is a Banach space. Now the signature 
and truncated signature are defined as follows.

\begin{defn}[Signatures] \label{def: sig}
The signature $\,S(X) \,$ of a path $X\in \V^p([0,T], E)$, $p \ge 1$ is defined by $S(X) := (1, X^1, X^2, ...)\in T((E))$, where the $k$-th element 
\begin{equation} \label{eq: signature}
X^k:= \int\cdots\int_{0<t_1<\cdots < t_n<T } {\mathrm d} X_{t_1}\otimes \cdots \otimes {\mathrm d} X_{t_n} \in E^{\otimes k}
\end{equation}
is the $k$-fold, iterated integral for $k \ge 1$, 
if the iterated integrals are well defined. 

The truncated signature is naturally defined as $S^m(X) := (1, X^1, X^2, ..., X^m)\in T^m(E)$ for every $m \ge 1$ including the $0$-th term $\,S^{0}(X) \, =\,  1 \,$. 
\end{defn}

\begin{rem}
	The integrals in \eqref{eq: signature} depend  on the nature of the paths. Here are some typical examples: 
	\begin{enumerate}
	\item If $X$ is of $1$-variation path, then the integrals \eqref{eq: signature} of the signature can be understood as the Stieltjes integral;
	\item If $X$ is of $p$-variation path with $1<p<2$, then it can be defined in the sense of Young (e.g., see \cite{lyons2002system}).
	\item If $X$ is a Brownian motion, then we can use the It\^{o} integral or the Stratonovtich integral. As we will explain later, when extending from a Brownian motion path or a semimartingale to a  geometric rough path, we choose the Stratonovitch integral rather than the It\^{o} integral.
	\end{enumerate}
\end{rem}

\begin{exam}[Smooth paths and piece-wise linear paths]
For $p \ge1$  the path space 
$\V^{p}([0,T], E)$ contains the smooth functions and the piece-wise linear functions. 
We give the following two examples of paths in $\V^p([0,T], E)$, as shown in Figure \ref{path_eg}. In its left panel, 
we plot the smooth path $X_t = (t, (t-2)^3), t\in[0,4]$. In its right panel, 
we represent the discrete data: daily AAPL adjusted close stock price from Nov 28, 2016 to Nov 24, 2017 by interpolating the path linearly between each successive two days. 
The first $\,2\,$ degree signatures $\, X^{1}\,$ and $\, X^{2}\,$ of these two paths in \eqref{eq: signature} are calculated and given in Table \ref{tab:example_sig}.

\end{exam}

\begin{table}[H]
\centering
\begin{tabular}{l|cc}
\hline
  $X$ & $(t, (t-2)^3)$ & AAPL \\ \hline
 $X^1$ &    $(4, 16)$    & $(1, 65.52)$ \\
$X^2$ & $\left(\begin{array}{cc} 8 & 32\\
				32 & 128\end{array}\right)$  &  $\left(\begin{array}{cc}0.5 & 31.17\\
				34.00 & 2123.3\end{array}\right)$ \\ \hline 
\end{tabular}
\caption{The corresponding signatures for the smooth path $(t, (t-2)^3)$ and the piece-wise linear path of the augmented AAPL adjusted price in Figure \ref{path_eg}, respectively.}
\label{tab:example_sig}
\end{table}

\begin{figure}[h]
\centering
	\subfloat[Smooth path]{
	\includegraphics[scale=0.35]{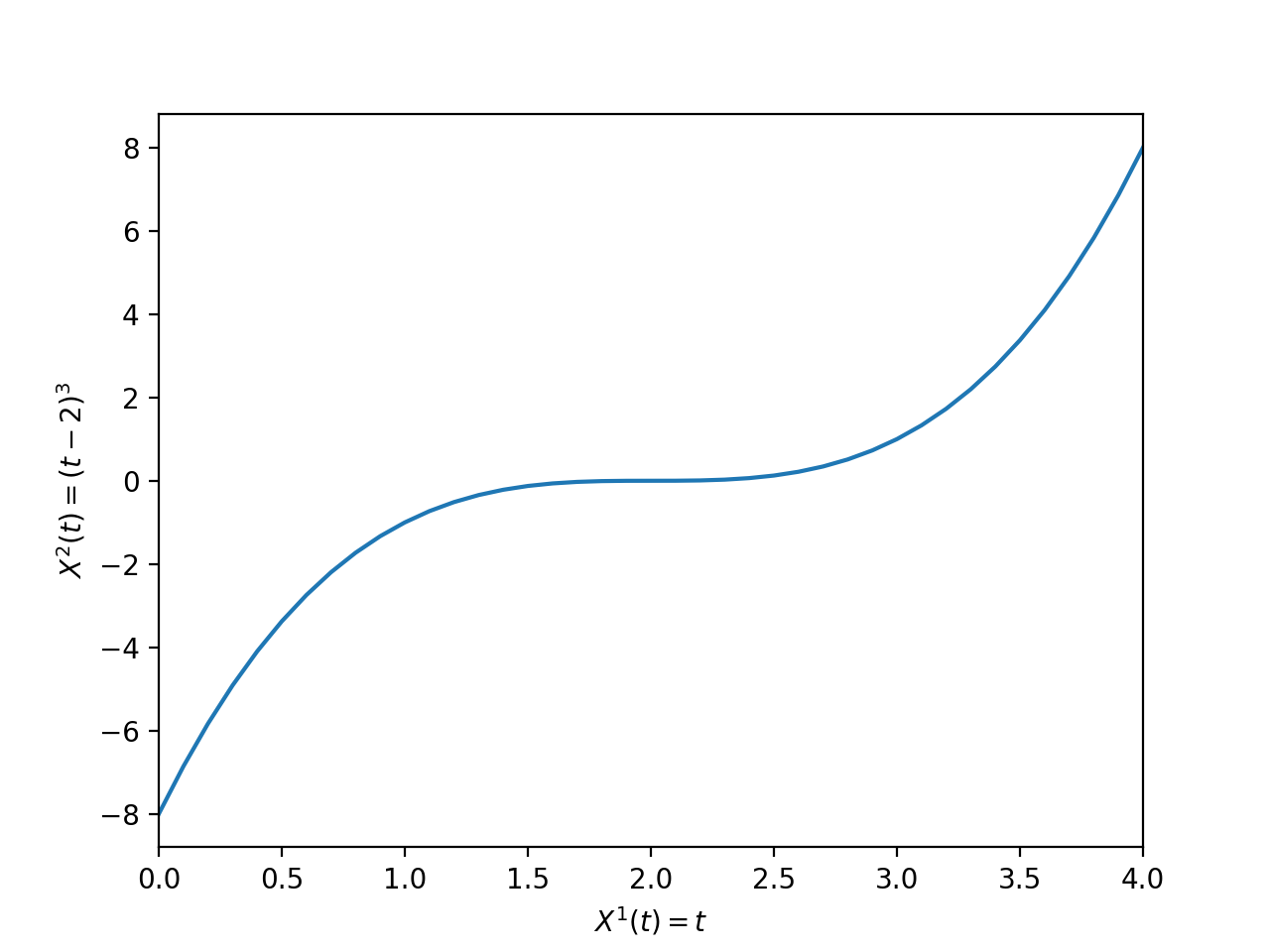}
	} 
	\subfloat[Piecewise linear path]{
	\includegraphics[scale=0.35]{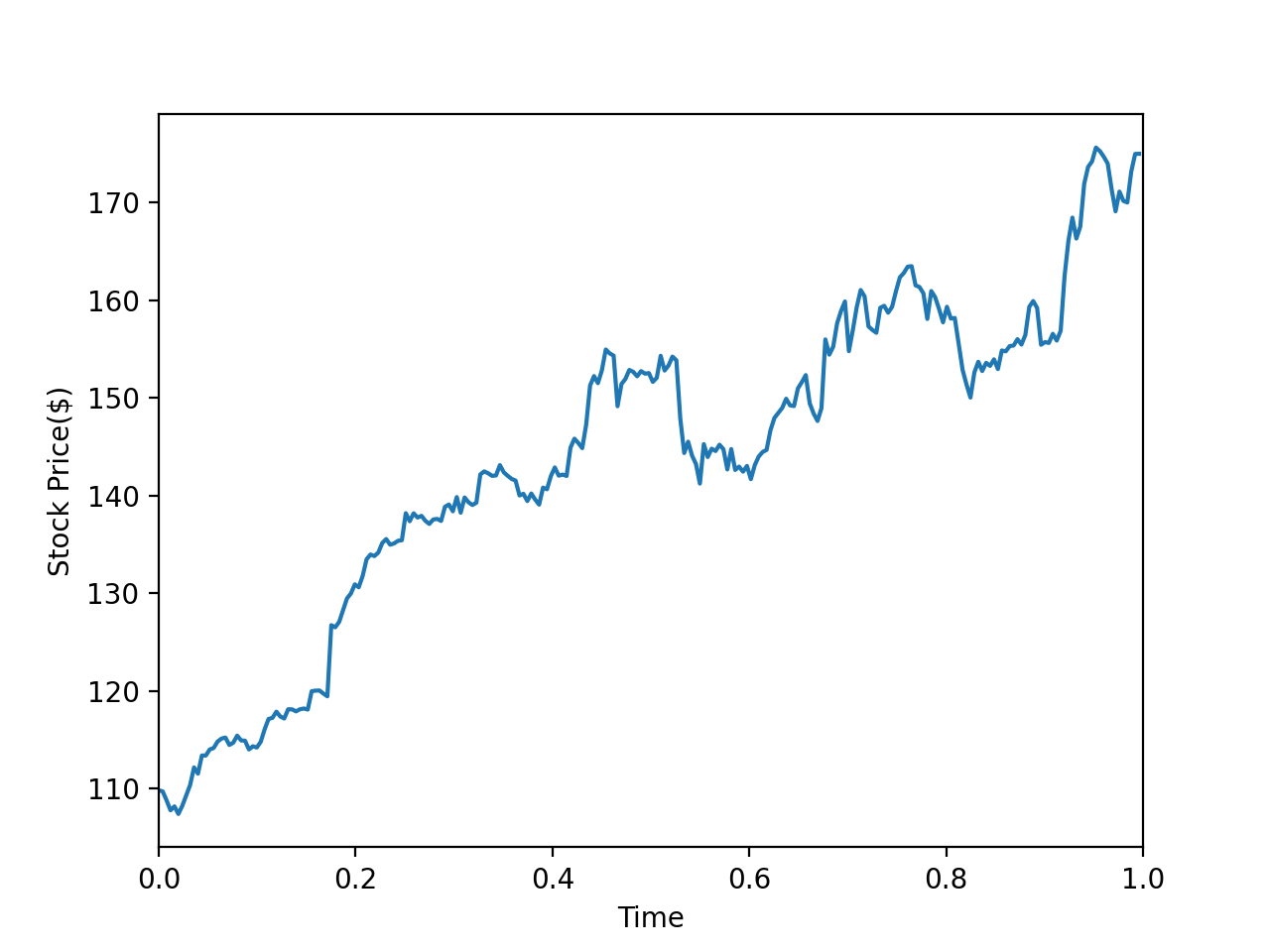}
	} 
	\caption{Examples of $\V^{p}([0,T], E)$, $p \ge1$: (a) Plot of a smooth path $X_t = (t, (t-2)^3), t\in[0,4]$. (b) Plot of linear interpolation of daily AAPL adjusted close stock price from Nov 28, 2016 to Nov 24, 2017.} 
	\label{path_eg}
\end{figure}


\subsection{Geometric Rough Paths and Linear Functionals} \label{sec: 2.2}

Here we introduce rough paths and geometric rough paths briefly. More details can be found in \cite{lyons2002system} and \cite{LyonsTerryJ2007DEDb}. Instead of $T((E))$ in \eqref{eq: whole}, the $p$-rough paths and the geometric $p$-rough paths are objects in $T^{\lfloor p \rfloor}(E)$ in \eqref{eq: truncatedbym} for some real number $p \, (\geq 1)$. A fundamental result from rough paths theory and signatures \cite{LyonsTerryJ2007DEDb} is that there exists a continuous unique lift from $T^{\lfloor p \rfloor}(E)$ to $T((E))$. This lift is made in an iterated integral,  and consequently, it gives us the signature of rough paths. 

We denote the space of the $p$-rough paths by $\Omega_p$. The space $G\Omega_p$ of the geometric $p$-rough paths is defined by the $p$-variational closure (\textit{cf.} \cite{LyonsTerryJ2007DEDb} Chapter 3.2) of $S^{\lfloor p\rfloor}(\Omega_1)$. For a path $X:[0,T]\to \R^d$ with the bounded $p$-variation, the truncated signature belongs to the space of the $p$-rough paths, i.e., $S^{\lfloor p\rfloor}(X)\in \Omega_p$. If $X$ is of bounded $1$-variation, then the truncated signature belongs to the space of the geometric $p$-rough paths, i.e., $S^{\lfloor p\rfloor}(X)\in G\Omega_p$ for any $p(\geq 1)$.

It is manifested that the signature enjoys many nice properties. For example, signature characterizes paths up to tree-like equivalence \cite{boedihardjo2014signature} that are parametrization invariant. Here is a precise statement. 
\begin{prop}[Parametrization Invariance, Lemma 2.12 of \cite{Levin2013LearningFT}] 
Let $X:[0,T]\to \R^d$ be a path with bounded variation and $\psi: [0,T]\to[0,T]$ a re-parametrization of the time parameter. If we define $\tilde{X}$ by $\tilde{X}_t := X_{\psi(t)}$, then each term in $S(\tilde{X})$ is equal to the corresponding term in $S(X)$, i.e. $S(\tilde{X}) = S(X)$.
\end{prop}

Moreover, if there exists a monotone increasing dimension in the path with bounded variation or geometric rough path, we can get rid of tree-like equivalence \cite{boedihardjo2014signature,gyurk2013extracting,Levin2013LearningFT}. 
Also, it is easy to specify one path among the parametrization invariance by adding timestamps. 
In other words, provided that an extra time dimension included, signature characterize geometric rough path uniquely. Another useful fact from rough path theory \cite{Chevyrev_2016,lyons2002system} 
is that signature terms enjoy a factorial decay as the depth increases, which makes truncating signature reasonable. 
The following remark shows an example of the factorial decay for bounded $1$-variation paths.

\begin{rem}[Factorial Decay, Proposition 2.2 of \cite{LyonsTerryJ2007DEDb}]
\label{exam_decay}
Let $X: [0,T] \to \R^d$ be a continuous path with bounded $1$-variation, then  for every $\, k \ge 1\,$ 
	\begin{equation}
	\left\| \ \  \idotsint\limits_{0\leq t_1 < \cdots < t_k\leq T} d X_{t_1}\otimes \cdots \otimes d X_{t_k}\right\| \leq \frac{\|X\|_{1}^k}{k!},
	\end{equation}
	where $\|\cdot\|$ is the tensor norm. 
\end{rem}

All these properties motivate us 
to use the signature as a feature map in Data Science. We shall then define the linear forms on the signatures.

For simplicity, let us fix $E = \R^d$, and let $\{e_i\}_{i=1}^d$ ($\{e_i^*\}_{i=1}^d$, respectively) be a basis of $\R^d$ 
(a basis of the dual space $(\R^d)^*$ of $\R^d$, respectively). 
For every $n \in \N$ and indexes $(i_{1}, \ldots , i_{n}) \in \{1, \ldots , d\}^{n}$, $(e_{i_1}^*\otimes \cdots \otimes e_{i_n}^*)$ can be naturally extended to $(E^*)^{\otimes n}$ with the basis $(e_I^*=e_{i_1}^*\otimes \cdots \otimes e_{i_n}^*)$, 
and we call $I=i_1\cdots i_n$ a {\it word} of length  $n$. The linear actions of $(E^*)^{\otimes n}$ on $E^{\otimes n}$ extends naturally a linear mapping $(E^*)^{\otimes n} \to T((E))^*$ by
\begin{equation}e_I^*(\mathbf{a}) := e^*_I(a_n),\end{equation}
for every word $I$ and every element $\mathbf{a}=(a_0, a_1, \dots, a_n, \dots)\in T((E))$.

Let $A^*$ be the collection of all words of length $n$ for all $n\in \N$. Then $\{e_I^*\}_{I\in A^*}$ forms a basis of $T(E^{\ast}) = T((\R^d)^*)$. Let $I,J\in A^*$ be two words of lengths $m$ and $n$  with $I=i_1\cdots i_m$ and $J=j_1\cdots j_n$, respectively.
We say a permutation $\sigma$ in the symmetric group $\mathfrak{G}_{m+n}$ of $\{1, \ldots , m+n\}$ is a \textit{shuffle} of $\{1,\dots, m\}$ and $\{m+1,\dots, m+n\}$, if $\sigma(1)<\dots< \sigma(m)$ and $\sigma(m+1) <\dots< \sigma(m+n)$. We denote the collection of all \textit{shuffles} of $\{1, \ldots , m\}$ and $\{1, \ldots , n\}$ by $\textit{Shuffles}(m,n)$.

\begin{defn}[Shuffle Product]
For every pair $I = i_1\cdots i_m$, $J= j_1\cdots j_n$ of words of length $m$ and $n$, the shuffle product $e_I^* \shuffle e_J^*$ of $e_I^*$ and $e_J^*$ is given by 
\begin{equation}
\label{shuffle_prod}
e_I^* \shuffle e_J^* := \sum_{\sigma \in \text{\rm Shuffles}(m,n)} e^*_{(k_{\sigma^{-1}(1)} \cdots k_{\sigma^{-1}(m+n)} )},
\end{equation}
where 
 $k_1\cdots k_{m+n} = i_1\cdots i_m j_1\cdots j_n$.
\end{defn}

Denote $T((\R^d))^*$ as the space of linear forms on $T((\R^d))$ induced by $T((\R^d)^*)$. The shuffle product between $f,g \in T((\R^d))^*$ denoted by $f \shuffle g$ can be defined via natural extension of \eqref{shuffle_prod}, by the 
bi-linearity of $\shuffle$. It can be shown that $T((\R^d))^*$ is an algebra equipped with shuffle product and element-wise addition restricted to the geometric rough path space $S(\V^p([0,T], \R^d))$, see Theorem 2.15 of \cite{LyonsTerryJ2007DEDb}. The following proposition motivates us to use the signature as a feature map.

\begin{prop}[Universal Approximation]
Fix $p\geq 1$, a continuous function $f: \V^p([0,T], \R^d)\to \R$ of finite $p$-variation, and  a compact subset $K$ of $\V^p([0,T], \R^d)$. If $S(x)$ is a $p$-geometric rough path for each $x\in K$, then for every $ \epsilon>0$, there exists a linear form $l^\epsilon \in T((\R^d))^*$, such that
\begin{equation}\sup_{x\in K}| f(x) - \langle l^\epsilon, S(x) \rangle | < \epsilon.
\end{equation}
\label{prop_LinearEstimation}
\end{prop}
\begin{proof}
The proof follows directly from the uniqueness of signature transform for geometric rough paths and the Stone-Weierstrass theorem. See \cite{lyons2019nonparametric} and Theorem 4.2 in \cite{arribas2018derivatives} for more details.
\end{proof}

\begin{rem}[A curse of dimensionality] \label{rem: dimensions} By Definition \ref{def: sig}, the truncated signature $S^m(X)$ 
has a total of ${\bf d}_{m}\, :=\, \sum_{k=0}^m d^k = {(d^{m+1} - 1)}/{(d-1)}$ many terms for $m \ge 0$. The signature transform is an efficient feature reduction technique, when we have the $d$ dimensional path sampled with high frequency in time. However, when the dimension $d$ is large, the number of signature terms to be computed increases exponentially fast and makes the signature not easily applicable in practice. 
\end{rem} 

To our best knowledge at this time, only \cite{JMLR:v20:16-314} and \cite{CLSY20} introduce new algorithms of calculating the kernel of the signatures and \cite{toth2019bayesian} discuss the application of the kernel methods to fix this high dimensional problem. We introduce Convolutional Neural Network (CNN) to solve this problem in Section \ref{sec_cnnsig}.

\subsection{Classification via Signature}

Before we discuss the convolutional neural network in Section \ref{sec_cnnsig}, we consider the application of the signatures to  classification problems. In classification problems, we estimate the probability of an object belonging to each class. This estimation problem  for the sequential data classification can be solved via the signature. 

On a probability space $(\Omega, \mathcal{F}, \p )$ consider $k$ classes, {\it class $1$}, {\it class $2$}, $\dots$, {\it class $k$}, and $n$ paired independent data $(x^i, y^i)_{1\leq i\leq n}$, where each $x^i:[0,T]\to\R^d$ is the path data and the corresponding label $y^i\in\{1, \dots, k\}$ is the class which $x^{i}$ belongs to. We assume that the labels $y^{1}, \ldots , y^{n}$ are sampled from a common distribution and the conditional probability $\mathbb P ( x^{i} \in \cdot \, \vert \, y^{i}) $ of $x^{i}$, given the class $y^{i}$, is a common probability distribution for $i =1, 2, \ldots , n $. Since we often observe the path dataset at discrete time stamps and we use piece-wise linear interpolations to connect among them, it is reasonable to assume that each path $x$ in the dataset is of bounded $1$-variation. Hence, its signature $S(x^i)$ is a geometric $1$ rough path in subsection \ref{sec: 2.2}. 

\begin{defn}[Classification problem] Our sequential classification problem is stated as follows: given training data $(x^i, y^i)_{1\leq i\leq n}$, derive a classifier $g$ for predicting the labels for unseen data $(x, y)$. 
Let $p_j(x):=\p( y=j |x)$ for $j =  1, \dots, k$. Our goal is to estimate these conditional probability $p_j(x)$ by $\hat{p}_{j}(x) $ for the path $x$ of bounded $1$-variation and classify $x$ in the class $\text{\rm arg}\max_{j} \hat{p}_{j}(x)$  for $j = 1, \ldots, k $ as accurate as possible. 
\end{defn}

Since the signature $S(x)$ of $x$ determines the path $x$ uniquely, it is reasonable to consider the signature $S(x)$ and a nonlinear continuous function $g: T((\R^d))\to [0,1]^k$, such that
\begin{equation}
\label{classify_g}
g(S(x)) = \left( \hat{p}_1(x), \dots, \hat{p}_k(x) \right)^{\mathrm T},
\end{equation}
where $\hat{p}_j$'s are estimator of $p_j$'s, subject to $\sum_{j=1}^k\hat{p}_j(x) = 1$. Here, $\mathrm T$ represents the transpose of the vector. 

For practical use, we use the truncate signature transforms, thanks to the factorial decay property (Remark \ref{exam_decay}) of the signature. 
With the truncation depth $m$, we obtain  the estimate  
\begin{equation}
\label{func_g}
g(S^m(x)) = \left( \hat{p}_1(x), \dots, \hat{p}_k(x) \right)^{\mathrm T},
\end{equation}
where $g: T^m(\R^d)\to [0,1]^k$ is a nonlinear continuous function, 
and then the predicted label is given by 
\begin{equation}
\label{y_hat}
\hat{y} =\argmax_j \hat{p}_j(x).
\end{equation} 

\begin{defn}[Signature Classifier]
We call $h:T((\R^d))\to [0,1]$ of the form \eqref{classify_g} a signature classifier, where $T((\R^d))$ is the tensor algebra and $h$ is a nonlinear continuous function. Naturally, a truncated signature classifier of degree $m\in \N$ is $h:T^m(\R^d)\to [0,1]$ of the form \eqref{func_g}.
\end{defn}

{
In the simple case with only $2$ classes, {\it class $0$} and {\it class $1$}, we consider the following concentration inequalities for classification via signature. We first restate the classification problem for the two classes. Suppose we have the pairwise, independent, identically distributed samples $(X^1, Y^1), \dots, (X^n, Y^n)$ where $Y^i\in\{0,1\}$ and $X^i \in \mathcal{V}^1([0,T], \R^d)$. Let $h: \mathcal{V}^1([0,T], \R^d) \to \{0, 1\}$ be a classifier. The training error $\hat{R}_n(h) $ and the true error $R(h)$ are defined by 
\begin{equation} \label{eq: hatRnh}
\hat{R}_n(h) = \frac{1}{n} \sum_{i=1}^n I(Y^i \neq h(X^i))
\, , \quad \text{ and } \quad R(h) 
= \p(Y\ne h(X)) . 
\end{equation}
Here, $\,I (\cdot) \,$ is the indicator function. Correspondingly, $R(h) = \p(Y\ne I(h(X)>0.5))$ and $\hat{R}_n(h) = \frac{1}{n}\sum_{i=1}^n I(Y^i \neq I(h(X^i)>0.5))$.  
We shall see that $\hat{R}_n(\hat{h}):= \inf_{h \in \mathcal{H}} \hat{R}_n(h)$ is close to $R(h_*):= \inf_{h\in \mathcal{H}}R(h)$, where $\mathcal{H}$ is the collection of the signature classifiers and we assume that $h_* \in \mathcal{H}$. Denote the set 
$$\mathcal{E} := \{\sup_{h\in\mathcal{H}} |\hat{R}_n(h) - R(h)| \leq \epsilon\}$$
 to be the event that the training error $\hat{R}_{n}(h)$ is close to the true error $R(h)$ for all classifiers $h\in \mathcal{H}$ in the range of $\varepsilon$, given a fixed $\varepsilon > 0 $. 


From now on, we assume $\mathcal{H}$ is a compact set of truncated signature classifiers of degee $m$ equipped with metric $\rho$. The following definition comes from \cite{hd_probability}. 

\begin{defn}[$\delta$-net and covering number] 
\label{def:deltanet} 
A set $H$ is called a $\delta$-net for $(\mathcal{H}, \rho)$ if for every $h\in\mathcal{H}$, there exists $\pi(h)\in H$ such that $\rho(h, \pi(h))<\delta$. The smallest cardinality of a $\delta$-net for $(\mathcal{H}, \rho)$ is called the covering number

\begin{equation}
N(\mathcal{H}, \rho, \delta):= \inf\{|H|: H \text{ is a } \delta\text{-net for }(\mathcal{H}, \rho)\}.
\end{equation}
\end{defn}

In our case, we may take the uniform norm $\rho$, for example. Indeed, by the Ascoli-Arzel\`{a} theorem, we only need $\mathcal{H}$ to be equicontinuous to make it compact, and hence $N_\delta:=N(\mathcal{H}, \rho, \delta)$ is always finite for any $\delta>0$. Let $H_{\delta}$ be a $\delta$-net of $\mathcal{H}$ with cardinality $N_\delta$.

\begin{thm}
\label{thm_concentration}
For every $\epsilon>0$, $\epsilon_0>0$, there exist $\delta>0$ and a corresponding finite covering number $N_\delta$, such that
\begin{equation}
\label{ineq_concentration}
\p(\sup_{h\in \mathcal{H}} |\hat{R}_n(h) - R(h)|>\epsilon)\le 2 N_\delta\,  e^{-2n\epsilon} + \epsilon_0.
\end{equation}
\end{thm}

\begin{proof}
Take a $\delta$-net $H_{\delta}$ of $\mathcal H$ with cardinality $N_{\delta}$. By the Markov inequality and the definition of the covering number, we have
\begin{align*}
	\p(\sup_{h\in H_{\delta}} (\hat{R}_n(h) - R(h))>\epsilon) & \le e^{-t\epsilon} \E[\sup_{h\in H_{\delta}} e^{t (\hat{R}_n(h) - R(h))}] \\
	&\le N_{\delta} e^{-t\epsilon} \sup_{h\in H_{\delta}}\E[e^{t (\hat{R}_n(h) - R(h))}].
\end{align*}
Since $\hat{R}_{n}(h)$ is the sum \eqref{eq: hatRnh} of independent random variables, by Hoeffding's inequality \cite{hoeffding}, we have
$e^{-t\epsilon} \E[e^{t (\hat{R}_n(h) - R(h))}] \le e^{-2n\epsilon}$ for $h\in H_{\delta}$, $t \ge 0$ and $n\ge1$. 
Hence, for every $n \ge 1$ and $\delta$-net $H_{\delta}$ of $\mathcal H$, we have 
\begin{align*}
	\p(\sup_{h\in H_{\delta}} (\hat{R}_n(h) - R(h))>\epsilon) &\le N_{\delta} e^{-t\epsilon} \sup_{h\in H_{\delta}}\E[e^{t (\hat{R}_n(h) - R(h))}] \le N_\delta e^{-2n\epsilon}.
\end{align*}
By a similar argument, we also have 
$\p(\sup_{h\in H_{\delta}} ( R(h) -\hat{R}_n(h) )>\epsilon) \le N_\delta e^{-2n\epsilon}$ for every $n \ge 1$ and $\delta$-net $H_{\delta}$ of $\mathcal H$.  

Combining the above two inequalities, we obtain that for every $n \ge 1$ and $\delta$-net $H_{\delta}$ of $\mathcal H$ 
\begin{equation*}\p(\sup_{h\in H_{\delta}} |\hat{R}_n(h) - R(h)|>\epsilon) \le 2 N_\delta e^{-2n\epsilon}.\end{equation*}
By approximating the supremum over $\mathcal H$ by the supremum over the sets $H_{\delta}$ with cardinality $N_{\delta}$, that is, 
  \begin{equation*}\p(\sup_{h\in \mathcal{H}} |\hat{R}_n(h) - R(h)|>\epsilon) = \lim_{\delta\to 0}\p(\sup_{h\in H_{\delta}} |\hat{R}_n(h) - R(h)|>\epsilon),
\end{equation*}
we conclude \eqref{ineq_concentration} that for any $\epsilon_0>0$, there exits a $\delta>0$, 
\begin{align*}
\p(\sup_{h\in \mathcal{H}} |\hat{R}_n(h) - R(h)|>\epsilon) &< \p(\sup_{h\in H_{\delta}} |\hat{R}_n(h) - R(h)|>\epsilon)+\epsilon_0 \\
&\le 2 N_\delta e^{-2n\epsilon} + \epsilon_0.
\end{align*}
\end{proof}
By Theorem \ref{thm_concentration}, the event $\mathcal{E}$ holds with high probability provided that $n$ is sufficiently large. On the set $\mathcal{E}$, we have by definitions 
\begin{equation} \label{eq: 2.15}
R(h_*) \le R(\hat{h}) \le \hat{R}_n(\hat{h}) + \epsilon \le \hat{R}_n(h_*) + \epsilon \le R(h_*) + 2\epsilon . 
\end{equation}
Thus, it follows that $|R(\hat{h}) - R(h_*)|\le 2\epsilon$ on the set $\mathcal{E}$. Thus, 
on $\mathcal{E}$, the best empirical signature classifier $\hat{h}$ is close to the best true signature classifier $h_*$ as in \eqref{eq: 2.15}. The connection between signature classifier and general classifier can be constructed by the uniqueness of the  signature transform.

This covering number $N_\delta$ in Definition \ref{def:deltanet} plays an essential role here. The study of the covering number $N(\mathcal H, \rho, \delta)$ for the compact set $\mathcal H$ of the truncated signature classifiers is still in progress. If we can quantify this number, then the number of training samples $n$ needed for fixed error can be calculated from \eqref{ineq_concentration}.
}

\begin{exam}[GARCH time series]
\label{garch_exam}
We give an example of two classes of time series, $\{x^n\}_{n=1}^N$, generated by GARCH(2,2) model. The time series are given by 
\begin{align*}
& x^n_k = \sigma_k \epsilon_k, \\
& \sigma_k^2 = w + \sum_{i=1}^2\alpha_i x^n_{k-i} + \sum_{j=1}^2 \beta_j \sigma^2_{k-j},
\end{align*}
where $w>0$, $\alpha_i\geq 0$, $\beta_j \geq 0$ and $\epsilon_k$'s are I.I.D. standard normal distributed. Denote $ {\bm \alpha} = (\alpha_1, \alpha_2)$ and $ {\bm \beta} = (\beta_1, \beta_2)$. $2$ classes of GARCH time series are generated by setting parameters in Table \ref{garch_params}.

\begin{table}[h]
\centering
\begin{tabular}{c|ccc}
class & $w$   & $ {\bm \alpha}$      & ${\bm \beta}$       \\ \hline
1     & 0.5 & (0.4, 0.1) & (0.7, 0.5) \\
2     & 0.2 & (0.8, 0.5) & (0.4, 0.1)
\end{tabular}
\caption{Parameters for GARCH(2,2) time series.}
\label{garch_params}
\end{table}
For paths $x^n$ generated by the first row parameters in Table \ref{garch_params}, we label $y^n=1$ $($class $1$$)$, for the rest paths $x^n$ generated by the second row parameters in Table \ref{garch_params}, we label them by $y^n=2$ $($class $2$$)$. Thus, we generate paired data $\{(x^n, y^n)\}_{n=1}^N$.
\end{exam}

\begin{rem}
\label{class_rem}
It is important to note that we cannot directly apply  Proposition \ref{prop_LinearEstimation} here, because this $p(x)$ may not be continuous in $x$. Intuitively, it is  better to add nonlinearity on classifier $h(\cdot)$. The experiment in Section \ref{DirectedChain} verifies this intuition.
\end{rem}

In practice, the signature classifier \eqref{func_g} and its truncation \eqref{y_hat} can be applied to find the classification model $g(\cdot)$ to estimate $\hat{y}$ in other contexts. 
In Section \ref{sec_exp}, we shall apply the  logistic regression to Example \ref{garch_exam}, and the result shows that the use of the truncated signature to classify this GARCH(2,2) time series is significantly efficient. 


\section{Convolutional Signature Model}
\label{sec_cnnsig}
The main goal of this section is to introduce the Convolutional Signature (CNN-Sig) model. As we have seen in Remark \ref{rem: dimensions} in Section \ref{sec: 2.2}, the truncated signature suffers from the exponential growth of the number ${\bf d}_{m}$ of terms, when the dimension $d$ is large, and in this case both space and time complexity increase dramatically. We will use Convolutional Neural Network (CNN) to reduce this exponential growth to at most linear growth. CNN has been mostly used in analyzing visual imagery, where it takes advantage of the hierarchical patterns in image and assembles complex patterns by focusing on many small pieces of the picture. Convolutional layer convolves the input data with a small rectangular kernel, and the output data can be masked with an activation function. As there are some patterns between channels of a path, this motivates us to consider the  signature with CNN to address the high dimensional problem. 

Before introducing the CNN-Sig model, we shall explain that the signature transform can be viewed as a layer in the deep neural network model. 

\subsection{Signature as a Layer}
Signature transform can be viewed as a layer in deep neural networks and this is firstly proposed in \cite{NIPS2019_8574}. In the background of Python package {\bf signatory} \cite{signatory}, signature transform takes input tensor of shape $(b, n, d)$, corresponding to a batch of size $b$ of paths in $\R^d$ with $n$ observing points at times $\{t_j\}_{j=1}^n$, and returns a tensor of shape $(b, {\bf d}_{m})$ 
or a stream like tensor of shape $(b, n, {\bf d}_{m})$, where ${\bf d}_{m}$ is defined in Remark \ref{rem: dimensions} in Section \ref{sec: 2.2}. 
 Usually it omits the first term $1$ of the signature transform. Since the signature is also differentiable numerically with respect to each data points, the backpropagation calculation is available. In this way, the signature can be viewed as a layer in neural network.

\subsection{Convolutional Signature Model}
CNN, which has been proved to be a powerful tool in computer vision, is an efficient feature extraction technique. This idea has been used in \cite{liao2019learning} as well as the \textbf{``Augment"} module \cite{signatory} (but only $1$D CNNs are used). 
There are two cases of using $1$D CNNs. The first  case is to extract new sequential features of original paths and then paste them to the original path as extra dimensions. This method is not helpful in the high dimensional case and causes extra difficulty. The second case is that we use extracted sequential features directly from the $1$D CNN. It works as a dimension reduction technique but the challenge is that it causes loss of information. 

With the favor of the $2$D CNN, we are able to reduce the number of signature features and capture all information in the original path at the same time. Since the convolution here is different from the convolution concept in mathematics, we define it and present Example \ref{exam 2D CNN} to show the computational details for those who are not so familiar with CNN.

\begin{defn}[$2$D Convolution]
Let $*$ be an operation of element-wise matrix multiplication and summation between two matrices of the same shape, that is, $A := (a_{i,j})_{1\le i \le m,1\le j \le n}$ and $B := (b_{i,j})_{1\le i \le m,1\le j \le n}$ of the same size: 
$\, A * B = \sum_{i=1}^{m}\sum_{j=1}^{n} a_{i,j} b_{i,j} \, $. 
Suppose the input tensor is $M := (M_{i,j})_{1\le i\le I,1\le j \le J}$, a kernel window $K := (k_{i,j})_{1 \le i\le m,1\le j \le n}$ and a stride window $(s, t)$. The output $O:=(o_{p,q})$ of $2$D convolution is given by
\begin{equation}
o_{p, q} := (M_{i, j})_{1+(p-1)s\le i \le m +(p-1)s, 1+(q-1)t\le j\le n+(q-1)t } * K.
\end{equation}
\end{defn}

The shape of the output $O$ depends on how we treat the boundary specifically and does not play a crucial role here.

\begin{exam}[2D Convolution] 
\label{exam 2D CNN}
Let us consider a tensor $M := (M_{i,j})_{1\le i,j \le 5}$ and a kernel window  $K := (k_{i,j})_{1 \le i,j \le 3}$, 
\begin{equation*}
M := \left(\begin{array}{ccccc} 
2 & 1 & 0 & 2 & 0 \\
0 & 1 & 2 & 2 & 1 \\
0 & 0 & 0 & 1 & 1 \\
2 & 0 & 0 & 2 & 2 \\
0 & 2 & 0 & 1 & 1
\end{array}\right), \quad \text{ and } \quad 
K := \left(\begin{array}{ccc} 
0 & 1 & 0 \\
1 & 0 & -1 \\
-1 & -1 & -1
\end{array}\right), \text{ respectively}, 
\end{equation*} 
and a stride window $(1, 1)$. The output will be a $3\times 3$ tensor, denoted by $O=(o_{ij})_{1\leq i, j\leq 3}$, where each element $o_{i,j}$ of $O$ is given by the element-wise multiplication and summation of  
\begin{equation*} \widetilde{M}^{i,j} := (M_{k,\ell})_{i \le k \le i+2, j \le \ell \le j + 2} \end{equation*} and $ {K}$, i.e., $\, 
o_{i,j} = \widetilde{M}^{i,j} \ast K \, $ for $1 \le i , j \le 3$. For example, 
\begin{equation*}o_{11} = \left(\begin{array}{ccc} 
2 & 1 & 0 \\
0 & 1 & 2 \\
0 & 0 & 0  
\end{array}\right) * \left(\begin{array}{ccc} 
0 & 1 & 0 \\
1 & 0 & -1 \\
-1 & -1 & -1
\end{array}\right) = 2\cdot 0 + 1 \cdot 1 + 0\cdot 0 + \dots  + 0 \cdot (-1) = -1,\end{equation*}
\begin{equation*}o_{12} = \left(\begin{array}{ccc} 
 1 & 0 & 2\\
 1 & 2 & 2\\
 0 & 0 & 1
\end{array}\right) * \left(\begin{array}{ccc} 
0 & 1 & 0 \\
1 & 0 & -1 \\
-1 & -1 & -1
\end{array}\right) = 1\cdot 0 + 0 \cdot 1 + 2\cdot 0 + \dots  + 1 \cdot (-1) = -2,\end{equation*}

\begin{equation*}
\begin{split}
o_{13} = \left(\begin{array}{ccc} 
 0 & 2 & 0\\
 2 & 2 & 1\\
 0 & 1 & 1
\end{array}\right) * \left(\begin{array}{ccc} 
0 & 1 & 0 \\
1 & 0 & -1 \\
-1 & -1 & -1
\end{array}\right) = 1, \quad 
o_{21} = \left(\begin{array}{ccc} 
0 & 1 & 2 \\
0 & 0 & 0 \\
2 & 0 & 0
\end{array}\right) * \left(\begin{array}{ccc} 
0 & 1 & 0 \\
1 & 0 & -1 \\
-1 & -1 & -1
\end{array}\right) = -1, 
\end{split}
\end{equation*}
and so on. Therefore, the output $O$ is given by 
\begin{equation*}
O = \left(\begin{array}{ccc} 
-1 & -2 & 1 \\
-1 & -1 & -1 \\
0 & -5 &  -3 
\end{array}\right).\end{equation*}
\end{exam}

The Convolutional Signature model uses the $2$D CNN before the signature transform, and the structure of the convolutional signature model can be described in Figure \ref{CNN-Sig}. The convolution is implemented in channels. Since the signature is efficient in the time direction, we do not have to convolute the time direction. 

\begin{figure}[h]
	\centering
	\includegraphics[scale=0.4]{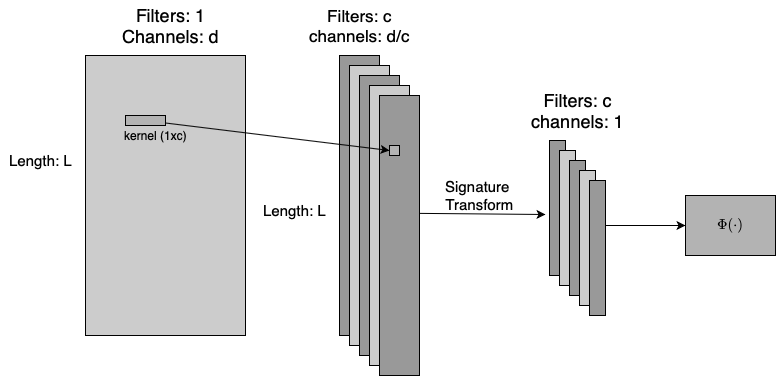}
	\caption{Convolutional neural network and signature transform connected by $\Phi$.}
	\label{CNN-Sig}
\end{figure}

\subsection{Number  of Features} \label{sec features}
Suppose $c\, (\leq d)$ is an integer such that $d$ is divisible by $c$ and let us fix the ratio $\gamma = {d}/{c}\in \N$. For the sake of simplicity of explanations, we set the number of features with kernel window of size $(1\times c)$ and stride $(1\times c)$. We illustrate our idea in the following example.

\begin{exam}
\label{exam convsig}
Let us consider a tensor $M := (M_{i,j})_{1\le i\le 5,1\le j \le 4}$ and $2$ kernel windows  $K_1 := (k^1_{i})_{1 \le i\le 2}$, $K_2 := (k^2_{i})_{1 \le i\le 2}$, 
\begin{equation*}
M := \left(\begin{array}{cccc} 
2 & 1 & 0 & 2  \\
0 & 1 & 2 & 2  \\
0 & 0 & 0 & 1  \\
2 & 0 & 0 & 2  \\
0 & 2 & 0 & 1 
\end{array}\right), 
\quad K_1 := \left(\begin{array}{ccc} 
-1 & 1
\end{array}\right)
\quad \text{ and } \quad 
K_2 := \left(\begin{array}{ccc} 
1 & 2
\end{array}\right).
\end{equation*}
By using a stride window $(1,2)$, we calculate the output $O = \{O_1, O_2\}$ with $O_l = (o^l_{i,j})_{1\le i\le 5, 1\le j\le 2}$, $l=1,2$. The computation is done in the same way as in Example \ref{exam 2D CNN}: 
$o^1_{1,1} = (2, 1)\ast (-1, 1)= -2+1 = -1$, 
$o^1_{2,2} = (2, 2)\ast (-1, 1) = -2 + 2 = 0$, 
$ o^2_{1,1} = (2, 1)\ast(1, 2) = 2 + 2 = 4$, 
$o^2_{2,2} = (2, 2)\ast(1, 2) = 2+4 = 6$.
Therefore, the output $O$ is given by
\begin{equation*}
O_1 = \left(\begin{array}{cc} 
-1 & 2  \\
1 & 0  \\
0 & 1  \\
-2 & 2  \\
2 & 1 
\end{array}\right),
\quad
O_2 = \left(\begin{array}{cc} 
4 & 4  \\
2 & 6  \\
0 & 2  \\
2 & 4  \\
4 & 2 
\end{array}\right).
\end{equation*}
In this example, since $K_1$ and $K_2$ are linear independent, we fully recover the input $M$ given $K_1, K_2$ and output $O$.
\end{exam}

Notice that since the first term in signature transform is always $1$, we can omit that, in order to save the computational memory. As shown in Figure \ref{CNN-Sig}, we start from one $d$-dimensional path with length $L$, by using such a convolutional layer, and we are resulted in $c$ paths with each of ${d}/{c}$-dimensional. Then we augment each path with extra time dimension and apply signature transform to each path truncated at depth $m$, which gives us the number of features 
\begin{equation}
N_{f} := c\cdot \frac{(d/c+1)^{m+1}-d/c-1}{d/c+1 - 1} = \frac{(\gamma+1)^{m+1}-\gamma-1}{\gamma^2} \cdot d 
\end{equation} 
many features by concatenating all $c$ filters. 
 These features can be used in any following neural network model. 
For example, a fully connected neural network in the simplest case, or a recurrent neural network (RNN) if we compute the sequence of the signature transform. 

The number $N_{f}$ of features grows linearly in $d$ by increasing $c$ linearly and fixing $\gamma$. Instead of optimizing this $N_{f}$ by setting $\gamma=\arg \min N_f$ directly, we can think $\gamma$ as a hyperparameter to be tuned to avoid overfitting problem. It can be easily seen that by setting $\gamma=1$, we reach a minimum of $N_f$ when $m\geq 3$. However, lower $\gamma$ will give us higher $c$, which increase the number of parameters in the CNN step. We consider the sum $N_f + (\frac{d}{\gamma})^2$ of number of features and the number of parameters in CNN.  
Moreover,  we can add a multiplier $\alpha$ to the second term, and then define a regularized number on $\gamma$,
\begin{equation}
\label{N_alpha}
N^\alpha(\gamma):=  \frac{(\gamma+1)^{m}-1}{\gamma^2}\cdot (\gamma + 1) \cdot d + \alpha \cdot \frac{d^2}{\gamma^2}.
\end{equation}
We can select a large real positive number $\alpha$. This will help us avoid the overfitting problem, when we are concerned about that the CNN layer fits the original paths too well and it sacrifices the prediction power.

\subsection{One-to-one Mapping}
Under the setup in Section \ref{sec features}, we can generalize Example \ref{exam convsig} and prove such a convolutional layer preserves all information of the original path. Suppose that $\{k^i\}_{i=1}^c$ are all $c$ convolutional kernels with $k^i=(k^i_1, \dots, k^i_c)$ for $i =1, \ldots , c$. Denote the square matrix \begin{equation*}\mathbf{K} := \left( \begin{array}{ccc}  
				k^1_{1} & \dots & k^1_{c} \\
				\vdots& \vdots& \vdots \\
				k^c_{1}& \dots & k^c_{c}
				\end{array} \right).\end{equation*}
Let the original path be 
$\mathbf{x} =\left( x_{t_1},\dots, x_{t_n} \right)^{\mathrm T}$, 
%
$x_{t_j} = \left( x_{t_j}^1, \dots, x_{t_j}^d \right)$ 
and the output path $\{\tilde{x}_i\}_{i=1}^c$, where $\tilde{x}_i=\left( \tilde{x}_{t_1, i},\dots, \tilde{x}_{t_n, i} \right)^{\mathrm T}$ with $\tilde{x}_{t_j, i} = \left( \tilde{x}_{t_j, i}^1, \dots, \tilde{x}_{t_j, i}^{\gamma} \right)$, $1\leq i\leq c$. The CNN layer can be represented in equation as 
\begin{equation}\label{eq:Kx}
\mathbf{K} \cdot \left( {x}_{t_j}^{lc+1}, \dots, {x}_{t_j}^{(l+1)c} \right)^{\mathrm T} = \left(\tilde{x}_{t_j, 1}^l,\dots, \tilde{x}^l_{t_j, c} \right)^{\mathrm T}, \ \ 1\leq l\leq \gamma,\ 1\leq j\leq n.\end{equation}

\begin{lem}
If $\mathbf{K}$ is of full rank, then this \text{CNN} layer is a one-to-one map.
\end{lem}
\begin{proof}
Since $\mathbf{K}$ is square and of full rank, it is invertible. 
\begin{equation*}\left( {x}_{t_j}^{lc+1}, \dots, {x}_{t_j}^{(l+1)c} \right)^{\mathrm T} =\mathbf{K}^{-1} \cdot  \left(\tilde{x}_{t_j, 1}^l,\dots, \tilde{x}^l_{t_j, c} \right)^{\mathrm T}, \ \ 1\leq l\leq \gamma,\ 1\leq j\leq n.\end{equation*}
If follows that the original path $\mathbf{x}$ can be fully recovered by $\tilde{x}:=\{\tilde{x}_i\}_{i=1}^c$. 
\end{proof}

We denote the CNN layer transform as $\mathbf{K}: \mathcal{V}^1([0,T], \R^d) \to \mathcal{V}^1([0,T], \R^{d/c+1})^c$. Here, plus $1$ in the dimension $(d/c) + 1$ comes from the time dimension we add to each convoluted paths. 

In accordance with practical case, we consider approximating functions with domain in a subspace of $\mathcal{V}^1([0,T], \R^d)$ that is observed at finite time stamps and connected by linear interpolation between consecutive points. More precisely, define
\begin{equation}
\begin{split}
\mathcal{V}^1_D([0,T], \R^d):= \{x \in \mathcal{V}^1([0,T], \R^d): \text{there exist } n\in \mathbb{N} \text{ and } 0=t_0 < \cdots < t_n = T \\
\text{ such that } 
x(t) = \frac{t_{i}-t}{t_{i} - t_{i-1}} x(t_{i-1}) + \frac{t - t_{i-1}}{t_{i} - t_{i-1}} x(t_{i}) \\
 \text{ for } t_{i-1} \le t \le t_{i}, i = 1, \ldots , n \}.
\end{split}
\end{equation}
Suppose $f:\mathcal{V}^1_D([0,T], \R^d) \to \R$ is the continuous function we need to estimate. Then we have the following theorem.

\begin{thm}[Approximation by the CNN-Sig model]\label{thm:App-CNN-Sig}
Let $K$ be a compact set in $\mathcal{V}_D^1([0,T], \R^d)$. Suppose that $f$ is Lipschitz in $K$. For any $\epsilon>0$ there exist a CNN layer $\mathbf{K}$, an integer $m$, and a neural network model $\Phi$ such that  
\begin{equation*}\sup_{x\in K}| f(x) - \Phi\circ S^m \circ \mathbf{K}(x) |<\epsilon.
\end{equation*}
\end{thm}

\begin{proof}
For every $ x\in \mathcal{V}_D^1([0,T], \R^d)$, we rewrite $f(x)$  as a function of $\tilde{x} = \{\tilde{x}_{i}\}_{i=1}^{c}$ in \eqref{eq:Kx}: 
\begin{equation}
\label{eq_cnn}
f(x) = f(\mathbf{K}^{-1}(\tilde{x})) = f\circ \mathbf{K}^{-1}(\tilde{x})=: h(\tilde{x}).
\end{equation}
It follows that $h = f \circ \mathbf{K}^{-1} $ is a continuous function. Since $S(\tilde{x}_i)$ is a geometric rough path and characterize the path $\tilde{x}_i$ uniquely for each $1\leq i\leq c$, there exists a continuous function $\hat{h}: (T(\R))^c \to \R$  
such that
\begin{equation*}h(\tilde{x}) = \hat{h}(S(\tilde{x}_1), \dots, S(\tilde{x}_c)).\end{equation*}
The existence follows from the compactness and that the signature map is continuous and one-to-one. Moreover, since $f$ is Lipschitz, we have that $h$ is Lipschitz and hence $\hat{h}$ is also Lipschitz.
The compactness of $K$ implies that the image of $S\circ\mathbf{K}$ is also compact, hence $h(\tilde{x})$ can be approximate arbitrarily well be truncated signatures up to a uniform truncation depth $m$ for all data in the set $K$. The existence of such $m$ is induced by the proof of \cite[Lemma 4.1]{pmlr-v139-min21a} and Lipschitz property. That is, there exists an integer $m$, such that 
\begin{equation}\label{eq_truncated_sig}
\sup_{x\in K}|\hat{h}(S(\tilde{x}_1), \dots, S(\tilde{x}_c)) - \hat{h}(S^m(\tilde{x}_1), \dots, S^m(\tilde{x}_c))| \le \frac{\epsilon}{2}.
\end{equation}

This $\hat{h}$ is not necessarily linear, because there might be some dependence among $\{\tilde{x}_i\}_{i=1}^c$, but it can be approximated by a neural network model arbitrarily well. A wide range of $\Phi$ can be chosen. For example, a fully connected shallow neural network with one wide enough hidden layer and some activation function would work, see \cite{FUNAHASHI1989183}, \cite{cybenko1989}; or a narrow but deep network, see \cite{pmlr-v125-kidger20a}. That is, there exists $\Phi$ such that 
\begin{equation}
\label{eq_neuralnetwork}
\sup_{x\in K}\left| \Phi(S^m(\tilde{x}_1), \dots, S^m(\tilde{x}_c)) -  \hat{h}(S^m(\tilde{x}_1), \dots, S^m(\tilde{x}_c)) \right|\le \frac{\epsilon}{2}.
\end{equation}
By combining \eqref{eq_cnn}, \eqref{eq_truncated_sig}, \eqref{eq_neuralnetwork} together, we get the desired result. 
\end{proof}

In the CNN-Sig model, the CNN layer can be understood as data dependent encoder which help us find the best way of encoding original path to several lower dimensional paths. On one hand, a large $c$ will result in overfitting problem of CNN layer. On the other hand, small $c$ will produce large number of features for $\Phi$, and then $\Phi$ may has the overfitting problem. This tradeoff can be balanced by minimizing $N^\alpha(\gamma)$ in equation \eqref{N_alpha}. Thus, although the choice of $c$ does not affect the universality of the model, it could help with resolving the overfitting problem.  

\begin{rem}
When we do experiments of the CNN-Sig model, this model works even better compare to plain signature transform of original path on testing data, it is because the CNN-Sig model reduces the number of features and thus overcome the overfitting problem better than direct signature transform.
\end{rem}

Moreover, the signature transform can be performed in a sequential way. Then we can choose a RNN model (GRU or LSTM) for $\Phi$. Some other candidates for $\Phi$ can be Attention model like Transformer, $1d$-CNN and so on, which might help us get better predictions. Thus, this CNN-Sig model is quite flexible and can be incorporated with many other well developed deep learning model as $\Phi$, which depends specifically on the task. In practice, we can use a different stride size to allow some overlap during convolution and reduce the number of filters. The one-to-one mapping property may be lost in this case if we choose small number of filters, but it results in less overfitting. Another alternative is that we can also convolute over time dimension, provided that correlation over time is of importance to the sequential data.


\section{Experiments}
\label{sec_exp}
In this section, several results of the experiments are provided for the purpose of exhibiting the performance of the signature classifier and the CNN-Sig model. Sections \ref{secCLGARCH} and \ref{DirectedChain} show that the signature classifier can be a nice candidate for the time series classification problem. In sections \ref{sec CCN-Sig regression} and \ref{secSAS}, we apply the CNN-Sig model to high-dimensional tasks, including the standard high-dimension datasets, approximation of maximum-call European payoff and sentimental analysis.

\subsection{Classification of GARCH Time Series}
\label{secCLGARCH}
The generalized autoregressive conditional heteroskedasticity (GARCH) process is usually used in econometrics to describe the time-varying volatility of financial time series \cite{bollerslev1986generalized,engle1982autoregressive}. GARCH provides a more real-world context than other models when predicting the financial time series, compare to other time series model like ARIMA.
We apply logistic regression to Example \ref{garch_exam}, i.e. the goal is to estimate $g(S^m(x)) = (\hat{p}_0, \hat{p}_1)$ in \eqref{func_g}, where 
\begin{equation}
\label{logis_reg}
\log \frac{\hat{p}_1}{1-\hat{p}_1} = \langle l, S^m(x) \rangle,
\end{equation} 
subject to $\hat{p}_0 + \hat{p}_1 =1$, $l$ is a linear functional on $T^m(\R^d)$ to be chosen such that the cross entropy 
\begin{equation}
E(l) = -\sum_{i=1}^N (y^i \log \hat{p}_i + (1-y^i)\log(1-\hat{p}_i))
\end{equation}
 is minimized, and we predict labels by $\hat{y}^i = \argmax_i \hat{p}_i$. $500$ samples are generated for each class and we use $70\%$ of each class as training data and $30\%$ of each as testing data.
By using $m=4$, we get training accuracy $96.4\%$ and testing accuracy $97.0\%$. The confusion matrix is given below in Table \ref{confusion}.

\begin{table}[h]
\centering
\begin{tabular}{|c|c|c|}
\hline
\backslashbox{True}{Predicted} & 0   & 1          \\ \hline
0         & 343 & 7 \\ \hline
1         & 18 & 332 \\ \hline
\end{tabular}
\quad \quad
\begin{tabular}{|c|c|c|}
\hline
\backslashbox{True}{Predicted} & 0   & 1          \\ \hline
0         & 147 & 3 \\ \hline
1         & 6 & 144 \\ \hline
\end{tabular}
\caption{Training (left) and testing (right) confusion matrics.}
\label{confusion}
\end{table}

\subsection{Classification of Directed Chain Discrete Time Series}
\label{DirectedChain}
In the study of mean-field interaction and financial systemic risk problems, \cite{detering2018directed} propose a countably many particle system of diffusion processes, coupled through an infinite, chain-like directed graph, and discuss a detection problem of mean-field interactions among diffusive particles. In Remark 4.5 of \cite{detering2018directed}, a discrete time analogue of the mean-reverting diffusions on the directed chain is also proposed. 

We shall discuss a classification problem of such time series data partially observed from the directed chain graph. More specifically, we analyze an identically distributed time series data $\{X_n\}_{n\geq 1}$ and $\{\widetilde{X}_n\}_{n\geq 1}$ parametrized by $a, u \in [0, 1]$ and defined recursively by 
\begin{equation}
X_{n} = a X_{n-1} + (1-a) ( u \widetilde{X}_{n-1} + (1-u) \mathbb E [ X_{n-1}] ) + \varepsilon_{n}, \quad n \ge 1, 
\end{equation}
where we assume that $X_{0} = \widetilde{X}_{0} = 0 $ for simplicity, the distribution of $\{X_{n}, n \ge 0 \}$ is identical with that of $\{ \widetilde{X}_{n}, n \ge 0 \} $ and $\varepsilon_{n}$, $n \ge 1$ are independent, identically distributed standard normal random variables, independent of $ \{ \widetilde{X}_{n}\}_{n \ge 1}$. 
The parameter $u\in[0,1]$ measures how much $X_n$ depends on its neighborhood and $1-u$ measures how much $X_{n}$ depends on the common distribution. 
$X$ and $\widetilde{X}$ have the same distribution with the moving average representation:
\begin{equation}
\label{directed}
\begin{split}
X_n & =  \sum_{0\leq l\leq k\leq n-1}\binom{k}{l} u^l (1-a)^l a^{k-l} \epsilon_{n-k, l}, \quad \\
 \widetilde{X}_n &=\sum_{0\leq l\leq k\leq n-1}\binom{k}{l} u^l (1-a)^l a^{k-l} \epsilon_{n-k, l+1}, \ \ n\geq 1, 
 \end{split}
\end{equation}
where $\{\varepsilon_{n,k}$, $n, k \ge 0\}$ is an independent, identically distributed array of standard normal random variables. 

Suppose that our only observation is $\{X_n\}_{n\geq 1}$, but  both $\{\widetilde{X}_n\}_{n\geq 1}$ and $u$ are hidden to us. Our question is that given the access to $\{X_n\}_{n\geq 1}$ generated by different $u$, can we determine their classes?

In this part, we first set the default parameters and generate training and testing paths according to \eqref{directed}.
First we initial some parameters: $a=0.5$, $u=0.2$ or $0.8$ for classification task, $N=100$ is the time steps, $1/N$ is the variance of $\epsilon$.In order to generate paths, we generate a $n \times (n+1)$ matrix $\mathcal{E}$ of the error terms $\epsilon$, and then pick the column we need for each $n$. The summation takes time $O(N^2)$ and we have to range $n$ from $1$ to $N$. The time complexity is the order of $O(N^3)$.
We simulate 2000 training paths and 400 testing paths for this task.

\bigskip
\noindent
\textbf{Method 1: Logistic Regression} 
In this method, we use 2000 training paths: 1000 for $u=0.2$ and 1000 for $u=0.8$. 
Calculating the signature transform of these paths, augmented with time dimension, up to degree 9, we build a Logistic Regression model on the signatures of training data and test this model, see equation \eqref{logis_reg}.

The result is shown in Table \ref{table_LR}. 
We observe that signature does capture useful features for $u$ in these special time series.

\begin{table}[H]
\centering
\begin{tabular}{l | l}
\hline
Training Acc & Testing Acc \\ \hline
0.7465       & 0.7375     \\ \hline
\end{tabular}
\caption{Training accuracy and testing accuracy on Logistic regression.}
\label{table_LR}
\end{table}
\vspace{0.2cm}

\smallskip
\noindent
\textbf{Method 2: Deep Neural Network} We build a Neural Network model in order to get a better result. We use $4$ hidden layers with $256, 256, 128, 2$ units respectively. For first 3 layers, we use "ReLu" as activation function, for last layer, we use "Softmax" activation function as the approximated probability values. After training for 20 epochs, the result is shown in Table \ref{table_NNsig}.

\begin{table}[H]
\centering
\begin{tabular}{l | l}
\hline
Training Acc & Testing Acc \\ \hline
0.8930       & 0.8925     \\ \hline
\end{tabular}
\caption{Training accuracy and testing accuracy on NN.}
\label{table_NNsig}
\end{table}

This 4 layer neural network model produces better accuracy than logistic regression. The reason follows Remark \ref{class_rem}. Logistic regression trains a linear classifier, but it cannot be used to estimate $p(\cdot)$ efficiently, because $p(\cdot)$ is not continuous in $x$. This DNN model add nonlinearity to $h(\cdot)$,s and hence works better.

\subsection{High Dimensional Time Series}
\label{sec CCN-Sig regression}
Signature is an efficient tool as a feature map for high frequency sequential data to reduces the number of features. However, the number of signature terms increases exponentially as dimension (or channels in the language of PyTorch) increasing. In Section \ref{sec_cnnsig}, we proposed the CNN-Sig model to address this problem. We test our model 
by applying it in both regression and classification problem.

\bigskip
\noindent
\textbf{Experiments - Regression Problem for Maximum-Call Payoff}

We investigate our model on a specific rainbow option, high-dimension European type maximum call option. In other words, we want to use our CNN-Sig model to estimate the payoff 
$$\max_{1\leq k \leq d}((X^k_T - K)^+),$$
where $T$ is terminal time, $K$ is strike price, superscript $k$ represents the $k$-th coordinate of this $d$-dimension path. If $X^k_T$ is smaller than $K$ for all $1\le k\le d$, this payoff is zero. Otherwise the payoff would be the maximum of $(X^k_T - K)$ over those $k$ satisfies $X^k_T\geq K$.
Result of this experiment may motivate us to use CNN-Sig model in high dimensional optimal stopping problem from financial mathematics.

Because of the limitation of exponential growth in the number of features, we use lower $d=6, 10, 12, 20$ to compare the performance between plain signature transform and CNN-Sig model. Then we apply this model to test its performance with higher dimension $d=50$.

\begin{table}[h]
\centering
\begin{tabular}{c|cccc|cccc}
\hline
   & \multicolumn{4}{c|}{Sig+LR}                                                              & \multicolumn{4}{c}{CNN-Sig}                                                           \\ \cline{2-9} 
   & \multicolumn{2}{c|}{Training}                         & \multicolumn{2}{c|}{Testing}     & \multicolumn{2}{c|}{Training}                         & \multicolumn{2}{c}{Testing}      \\ \cline{2-9} 
d  & \multicolumn{1}{c|}{MAE} & \multicolumn{1}{c|}{$R^2$} & \multicolumn{1}{c|}{MAE} & $R^2$ & \multicolumn{1}{c|}{MAE} & \multicolumn{1}{c|}{$R^2$} & \multicolumn{1}{c|}{MAE} & $R^2$ \\ \hline
6 ($\gamma=2$)  &            0.001       &            1.000                &          0.101         &   0.538    &             0.020     &          0.986           &            0.030        &    0.972   \\
10($\gamma=2$) &        0.000      &         1.000        &      0.124        &   0.806     &             0.033             &          0.988      &          0.062             &    0.962   \\
12($\gamma=2$) &            0.000        &     1.000          &        0.153       &  0.821     &               0.048           &           0.981       &         0.111           &   0.924    \\
20($\gamma=1$) &         0.000        &         1.000       &       0.225             &    0.838   &        0.177                  &             0.916               &             0.203           &    0.892    \\ \hline
\end{tabular}
\caption{Training and testing mean absolute error(MAE) and $R^2$ for the direct signature transform plus linear regression (Sig+LR) and the CNN-Sig model with $\Phi$ as a fully connected neural network.}
\label{hd_lr}
\end{table}

We generate $1000$ training paths and $1000$ testing paths for cases of $d=6,10,12$, and generate $3000$ training paths and $1000$ testing paths for case $d=20$. All stock price paths follows Black-Scholes model.

For all $4$ cases, we consider $m=4$ as the signature depth. For $\Phi$ in the CNN-Sig model, we use the same structure, $2$ fully connected layers followed by ReLu activation function and then a fully connected layer. We did not apply any technique for avoiding overfitting problem in the CNN-Sig model to make this comparison fair. The result for comparison is shown in Table \ref{hd_lr}. We can see that for all these 4 cases, the CNN-Sig model beat direct signature transform. Since the CNN-Sig model reduce the number of features, it can help avoid overfitting problem compare to Sig+LR. We produce the QQ plots for training and testing results of the CNN-Sig model, see Figure \ref{QQ_low}.

\begin{figure}[h]
\centering
	\subfloat[train d=6]{\includegraphics[width=1.5in]{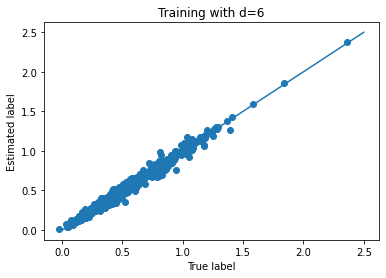}} 
	\subfloat[test d=6]{\includegraphics[width=1.5in]{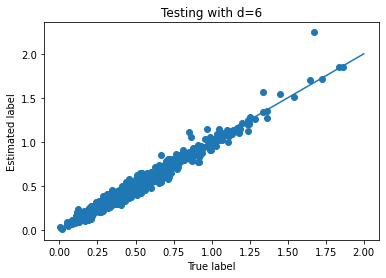}}\\
	\subfloat[train d=10]{\includegraphics[width=1.5in]{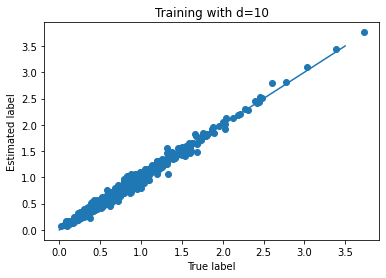}}
	\subfloat[test d=10]{\includegraphics[width=1.5in]{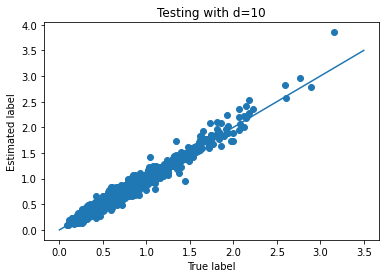}} \\
	\subfloat[train d=12]{\includegraphics[width=1.5in]{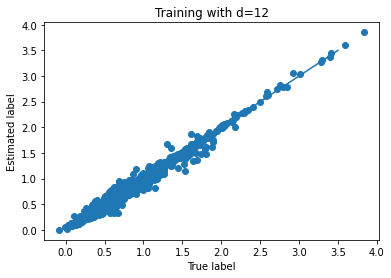}}
	\subfloat[test d=12]{\includegraphics[width=1.5in]{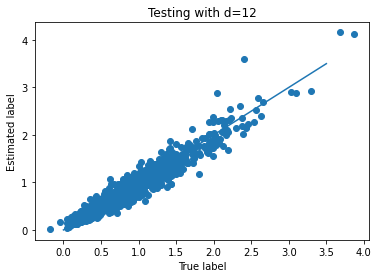}}\\ 
	\subfloat[train d=20]{\includegraphics[width=1.5in]{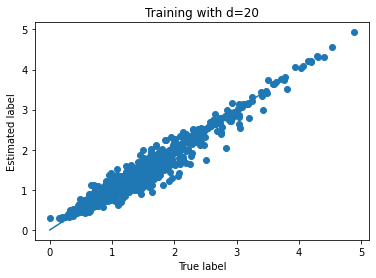}}
	\subfloat[test d=20]{\includegraphics[width=1.5in]{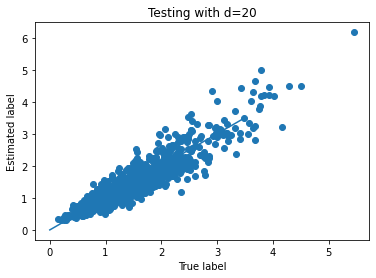}}\\
	\caption{QQ plot for training and testing result for lower dimensional regression with $d=6,10,12,20$ using the CNN-Sig model.}
	\label{QQ_low}
\end{figure}

For $d=50$, where the plain Sig+LR becomes not applicable, we use the same CNN-Sig structure as lower $d$ cases for training. The training MAE is $0.206$ with $R^2=0.982$ and testing MAE is $0.751$ with $R^2=0.797$. The QQ plot of training and testing results is in Figure \ref{QQ_high}. In this experiment, we show that CNN-Sig algorithm could be a good candidate in the high dimensional regression problem where plain signature is not applicable. But since CNN-Sig will add non-linearity here, we are not able to price this option in the same way as \cite{arribas2018derivatives}. This will be left as our future research.

\begin{figure}[h]
\centering
	\subfloat[d=50]{\includegraphics[width=2.3in]{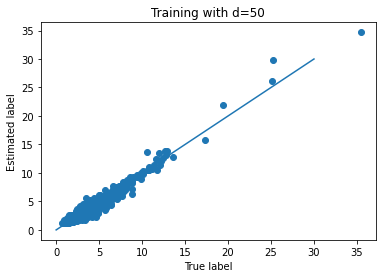}} 
	\subfloat[d=50]{\includegraphics[width=2.3in]{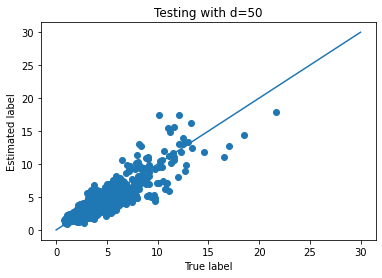}}
	\caption{QQ plot for training and testing result for regression task with $d=50$ using CNN-Sig model.}
	\label{QQ_high}
\end{figure}

\bigskip
\noindent
\textbf{Experiments - Classification}

We apply the CNN-Sig model to different high dimensional times series from \cite{baydogan2015dataset} and \cite{tsclass_2020}. As suggested in \cite{tsclass_2020}, all experiments are compared with a benchmark model ROCKET \cite{rocket}. The results are evaluated over 5 independent trials and listed in Table \ref{hd-ts}. ROCKET is known to be a fast and accurate classification method, the experiment results show that the CNN-Sig model is competitive and fast after a model selection procedure via $k$-fold cross validation.\footnote{All experiments are trained on a server with Intel Core i9-9820X (3.30GHz) and four RTX 2080 Ti GPUs} 


\begin{table}[h]
\centering
\begin{tabular}{lcc}
\hline
Datasets      & ROCKET & CNN-Sig \\ \hline
PEMS-SF & 0.810(0.014) & {\bf0.817(0.010)}  \\
JapaneseVowels        & {\bf 0.960(0.002)}        & 0.940(0.017)           \\
FingerMovement     & 0.500(0.01)        & {\bf0.514(0.034)}       \\
FaceDetection & {\bf0.597(0.004)} & 0.553(0.001)  \\
PhonemeSpectra & 0.035(0.002)        & {\bf0.152(0.006) }      \\
MotorImagery         & {\bf0.620(0.007) }      & 0.524(0.05)       \\
Heartbeat   & {\bf0.729(0.011) }       & 0.723(0.017)   \\ \hline
Training Time & 353.5 & {\bf209.1} \\  \hline 
\end{tabular}

\caption{Testing accuracy, standard deviation and total training time (s) for all high dimensional time series datasets.}
\label{hd-ts}
\end{table}

\subsection{Sentiment Analysis by Signature}
\label{secSAS}
In Natural Language Processing (NLP), text sentence can be regarded as sequential data. A conventional way to represent words is using high dimensional vector, which is called word embedding. These kind of word embedding is usually of $50, 100, 300$ dimension. Using plain signature transform becomes extremely difficult because of these high dimensions.
We apply our CNN-Sig model to address this problem. The dataset we use is IMDB movie reviews, \cite{maas-EtAl:2011:ACL-HLT2011}.

This IMDB dataset contains 50,000 movie reviews, each of them is labelled by either "pos" or "neg", which represent \textbf{Positive} for \textbf{Negative} respectively. The IMDB dataset is split into training and testing evenly. For training part, we use 17500 samples for training the model, and use the other 7500 samples as validation dataset. A 100-dimension word embedding GloVe 100d \cite{pennington2014glove} is used as the initial embedding, this high dimension restricts us to use plain signature transform. In our model, by setting $\gamma$ to be small, we use 1 convolutional 2d layer to reduce the dimension from 100 to $c$ paths with each of $\gamma+1$ dimensional augmented by extra time dimension. 
The architecture is shown in Figure \ref{imdb_nn}.

\begin{figure}[h]
	\centering
	\includegraphics[width=12cm]{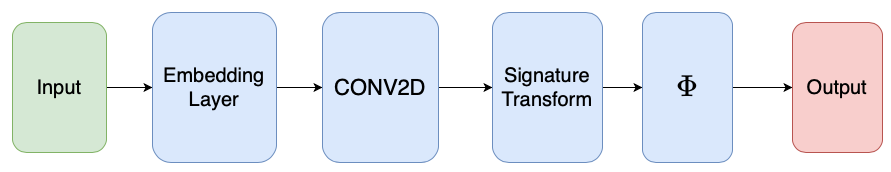}
	\caption{Convolutional Signature neural network model for IMDB dataset.}
	\label{imdb_nn}
\end{figure}

The result is shown in Table \ref{imdb} and the testing accuracy has been improved to {\bf 86.9\%} which is higher than the result in \cite{toth2019bayesian} (83\%) and Bidirectional LSTM (Bi-LSTM) with 2 hidden layers (0.846\%). Moreover, CNN-Sig is a more efficient structure compare to Bi-LSTM in terms of training time and GPU memory usage.

\begin{table}[H]
\centering
\begin{tabular}{lcc}
\hline
  & Bi-LSTM & CNN-Sig \\ \hline
 Accuracy &    0.846(0.013)    & {\bf 0.869(0.002)} \\
Memory &  6.8 &  {\bf1.3}  \\ 
Time & 401.5 & {\bf292.5}\\ \hline
\end{tabular}
\caption{Testing accuracy, GPU memory usage(Gb) during training and total training time(s) on IMDB dataset.}
\label{imdb}
\end{table}

We believe that the CNN-Sig model is a good candidate for feature mapping and easy to be embraced into more complex models. By applying more complicated structure, such as using attention model for $\Phi$ and a sliding window, e.g., see \cite{morrill2020generalised}, for calculating a sequential signature transform, the accuracy can be improved.

\section{Conclusion}
\label{sec_conc}
Using the signature to summarize sequential data has been proved to be very efficient in the low dimensional cases. However, signature transform suffers from exponential growth of the number of features with respect to the path dimension. This makes both regression and classification problem impossible in practice.

In this paper, we proposed the Convolutional Signature (CNN-Sig) model to address this problem. By using a convolutional layer, we achieve a linear growth of the number of features and preserve all information simultaneously. The experiments show that this model can be a good candidate for classifying multi-dimension sequential data. Moreover, signature has been proved experimentally to be insensitive to missing values, this property may be useful in many natural language processing (NLP) tasks. The CNN-Sig model mitigates the high dimension problem and provides a possible way to apply the signature transforms.


\bibliographystyle{spmpsci}

\end{document}